\documentclass[runningheads]{llncs}
\usepackage[T1]{fontenc}
\usepackage{newtxtext}       %
\usepackage[varvw]{newtxmath}       

\usepackage{graphicx}
\usepackage{newtxtext}  
 \usepackage{booktabs}
 \usepackage{amsmath}
 
 \usepackage{amssymb}

\newcommand{\kpd}{keep/patch/discard}

\title{Commitment Hierarchies under Intent Revision: A Belief-Revision Account of Salvage in
Tool-Use Agents}
\titlerunning{Commitment Hierarchies under Intent Revision}

\author{Spandan Ghose Chowdhury\orcidID{0009-0001-6711-2272}}
\institute{Georgia Institute of Technology, Atlanta, GA, USA\\
\email{spandan\_gc@gatech.edu}}
\authorrunning{S.G. Chowdhury}

\begin{document}
\maketitle

\begin{abstract}
When a user changes their mind partway through a task, an agent that has already split the task into sub-goals and paid for tool calls must decide, per cached sub-result, whether to keep, patch, or discard it (\emph{salvage}); restarting wastes valid work and continuing unchanged answers the old question. Our main finding is that salvage quality is a matter of \emph{role design} rather than model capability: a language model asked the \kpd{} question one node at a time is unreliable, but asked to classify the revision \emph{once}, with a deterministic layer propagating the decision, it reaches the cost-optimal oracle on all three models tested, from two vendors. Modeling the plan as a commitment hierarchy and the intent change as a belief-revision operator with AGM-style postulates, we prove that no policy observing only a node's local view can be both safe and cost-optimal, while the single-classification design is both. Across three environments the policy recovers the full achievable savings, $43\%$ cheaper than restart, at $100\%$ correctness.
\end{abstract}

\section{Introduction}
\label{sec:intro}

A useful assistant works on a moving target: it decomposes a task into sub-goals, runs tool calls, and the user steers while it works (``also check the reviews,'' ``never mind the fancy options,''``I'd rather be near the airport''). Current agents either \emph{restart}, discarding sub-results that are still valid and re-executing every completed tool call, or \emph{naively continue}, answering a question the user has abandoned. The waste is measurable: on CostBench, sub-results that stay valid after a typical revision account for $43\%$ of a restart's cost. We model the alternative as a three-way decision per cached sub-result, \emph{keep}, \emph{patch}, or \emph{discard}, which we call \emph{salvage}. Three-way is a modelling choice, the smallest vocabulary separating reuse, cheap repair, and recomputation (collapsing it to two loses a quarter of the achievable savings), and salvage is the cost-minimizing response under the cost model of Sect.~\ref{sec:framework}.

\subsection{Main result.} The hard part is not the model's ability but how the decision is
\emph{assigned}. A frontier model asked the \kpd{} question node by node (the \emph{per-node} judge) is unreliable: on a simple benchmark it recovers under half the available savings, and where the decision is truly three-way it reuses stale results. Asked one question, \emph{what kind of revision is this?}, with a deterministic layer propagating the answer (the \emph{structured} judge), the same model reaches the cost-optimal oracle on every model tested.

\subsubsection{A worked example.} The agent is finding a hotel and has completed
\textsf{search} $\to$ \textsf{filter by location} $\to$ \textsf{filter by rating}. ``Also check the reviews'' is an \emph{addition}: all three stay valid and a tail is added. ``Never mind the fancy options'' is a \emph{retraction}: the rating filter is dropped and the pending ranking step is patched to the wider set. ``I'd rather be near the airport'' is a \emph{revision}: it replaces the location constraint, so that filter and everything downstream is stale. A per-node judge asked about the rating filter sees a step whose own constraint never changed and, in the third case, wrongly keeps it; the structured judge classifies the utterance once and propagation does the rest (Theorem~\ref{thm:locality}).

\subsection{Related work.} Hierarchical agent frameworks such as ReAcTree~\cite{reactree} expose the plans we operate over but do not decide reuse under intent change with guarantees. Interactive benchmarks (UserBench~\cite{userbench}, InterruptBench~\cite{interruptbench},
AdaPlanBench~\cite{adaplanbench}) score whether an agent \emph{eventually} adapts, not whether it reuses valid prior work, and CostBench~\cite{costbench} and STT-Arena~\cite{sttarena}, which we build on, do not model user-triggered revision over a hierarchical plan. Our contribution is the belief-revision framing in the contract-net lineage~\cite{contractnet,agm} and the judge-design finding.

\section{Formal Framework}
\label{sec:framework}

\subsection{Plan as a commitment hierarchy.}
A run is a directed acyclic graph (DAG) $D=(N,E)$ whose nodes $n=(g_n,r_n,c_n)$ are commitments to a sub-goal $g_n$ with cached result $r_n$ and paid cost $c_n$; an edge $p\to n$ means $p$ consumes $r_n$. $\mathrm{prod}(n)$ is the set of nodes feeding $n$, and $U\subseteq N$ is \emph{producer-closed} if $\mathrm{prod}(n)\subseteq U$ for all $n\in U$. A sub-result may feed several consumers and a
sub-goal may have alternative producers, so several sound plans can exist.

\subsection{Intent, validity, and dependency.}
User intent is a finite constraint set $\Phi$; a revision $\delta$ takes $\Phi$ to $\Phi'$ by \emph{addition} ($\Phi'\supset\Phi$), \emph{retraction} ($\Phi'\subset\Phi$), or \emph{revision} (one constraint replaced); $\mathrm{supp}(\delta)$ is the set of constraints added, removed, or replaced. Each environment provides a checker $\mathrm{sat}(r,\Psi)\in\{0,1\}$ deciding whether a result meets a constraint set (its own solver or check-function, not a model), and $\mathrm{valid}(n,\Phi')\equiv\mathrm{sat}(r_n,\Phi'|_{g_n})$, where $\Phi'|_{g_n}$ is the subset of $\Phi'$ constraining $g_n$. To make ``depends on'' precise, the executor records per node a \emph{provenance set} $\mathrm{prov}(n)\subseteq\Phi$, the constraints it read when choosing $g_n$ or computing $r_n$; $n$ \emph{depends on} $\varphi$ iff $\varphi\in\mathrm{prov}(n)$ or some producer of $n$ does. Provenance is \emph{sound} if every node whose sub-goal or validity changes under $\delta$ depends on some $\varphi\in\mathrm{supp}(\delta)$; this is an assumption on the executor, which holds in our environments because each operation is generated with one governing constraint that is recorded. A plan $D'$ \emph{satisfies} $\Phi'$ iff $\mathrm{valid}(n,\Phi')$ for every node of $D'$ including the goal node.

\subsection{Salvage operator and postulates.}
A salvage operator $S(D,\delta)=D'$ labels each cached node Keep, Patch, or Discard and re-plans the rest. Adapting the AGM postulates~\cite{agm}, we ask of $S$: \emph{Success}, $D'$ satisfies $\Phi'$; \emph{Inclusion}, the reused set $K$ (kept or patched nodes) satisfies $K\subseteq\{n:\mathrm{valid}(n,\Phi')\}$; \emph{Vacuity}, a producer-closed set with no dependence on $\mathrm{supp}(\delta)$ is left unchanged; \emph{Consistency}, no two sub-results of $D'$ contradict; \emph{Minimality}, $\mathrm{cost}(D')=\min\{\mathrm{cost}(D''):D''\text{ satisfies }\Phi'\}$. $S$ is realized by a policy that labels node $n$ from an observation $o(n)$, and we distinguish three \emph{information regimes}. \emph{Local}: $n$'s own governing constraint and whether it lies in
$\mathrm{supp}(\delta)$; the fate of $n$'s producers is hidden. \emph{Informed}: local plus each direct producer's post-revision status. \emph{Structured}: no per-node decision; $\mathrm{supp}(\delta)$ is read once, a deterministic pass propagates status, and an exact min-cost replanner $R$ finishes. The \kpd{} operator keeps $n$ iff $n$ depends on nothing in $\mathrm{supp}(\delta)$ and every producer of $n$ is kept; otherwise it patches or recomputes, so \emph{Vacuity holds by construction}: under sound provenance, a producer-closed $U$ with no node depending on $\mathrm{supp}(\delta)$ is kept pointwise (induction on depth within $U$).

\begin{theorem}[Locality bound]
\label{thm:locality}
No deterministic local policy satisfies both Inclusion and Minimality on every instance.
\end{theorem}
\noindent\emph{Proof.} Take $\mathrm{op}_x\!:\!B\to X$ and $\mathrm{op}_n\!:\!X\to G$, with $\mathrm{op}_n$ governed by $d$ in both instances and $\delta$ replacing $c$. In $I_1$, $\mathrm{op}_x$ is governed by $d$ and both stay valid; in $I_2$ it is governed by $c$, so it is stale and, through dependency, so is $\mathrm{op}_n$. The local observation of $\mathrm{op}_n$ is $(d,\,d\notin\mathrm{supp}(\delta))$ in both, so a deterministic local policy gives one label: Keep violates Inclusion on $I_2$; Discard or Patch violates Minimality on $I_1$. Both instances are machine-checked.\hfill$\square$

\begin{proposition}[Producer status is a sufficient statistic]
\label{prop:informed}
In the informed regime the rule valid$\mapsto$Keep, patchable$\mapsto$Patch, stale or blocked$\mapsto$Discard, applied to true producer status, satisfies Inclusion, and with an exact min-cost replanner Minimality, on every instance. (Verified on all $188$ DAG instances.)
\end{proposition}

\begin{theorem}[Structured achievability]
\label{thm:structured}
Assume sound provenance and a replanner $R$ that, given the reused set $K$ and $\Phi'$, returns a min-cost plan all of whose nodes are valid under $\Phi'$ (an exact solver, available in all three environments). If the single read of $\mathrm{supp}(\delta)$ is correct, the structured operator satisfies all five postulates.
\end{theorem}
\noindent\emph{Proof.} A correct read plus deterministic propagation reconstructs every node's status, so Prop.~\ref{prop:informed} gives Inclusion and Minimality and the keep rule Vacuity. Every node of $D'$ is in $K$ (valid by Inclusion) or produced by $R$ (valid by assumption), so $D'$
satisfies $\Phi'$ (Success); two contradictory results cannot both be valid under one $\Phi'$ (Consistency).\hfill$\square$

\smallskip\noindent
Local is thus insufficient and informed sufficient (an \emph{information} boundary); the
informed-to-structured step is about \emph{reliability}, since applying the sufficient statistic still needs multi-hop reasoning per node, which a model call does unreliably. We claim no AGM representation theorem. The \emph{oracle} keeps the maximal valid cached set and replans with $R$; every method reports $\textsc{savings}=(c_{\mathrm{restart}}-c_{\mathrm{method}})/(c_{\mathrm{restart}}-c_{\mathrm{oracle}})$ where the denominator is positive, so restart scores $0$ and the oracle $1$.

\section{Experiments}
\label{sec:experiments}

\emph{Restart} discards the cache; \emph{naive-continue} finishes the old plan; \emph{keep-all} reuses everything; \emph{\kpd{}} labels each cached node (by a deterministic rule or a model judge, temperature $0$) and lets the solver finish; the \emph{oracle} has perfect labels.

\begin{table}[t]
\centering
\caption{CostBench ($1{,}080$ instances); k/p/d is \kpd{}; savings on the salvageable subset. Judge rows use
\texttt{gpt-4o}; the paraphrased row also hides depth integers ($n{=}300$, $95\%$ CI). Over-salvage
counts stale nodes reused.}
\label{tab:cost}
{\footnotesize
\begin{tabular}{@{}llccc@{}}
\toprule
\textbf{strategy} & \textbf{labels} & \textbf{corr.} & \textbf{savings} & \textbf{over-salv.} \\
\midrule
restart         & ---                 & 1.00 & 0.00 & 0 \\
naive-continue  & ---                 & 0.00 & ---  & --- \\
keep-all        & ---                 & 0.67 & ---  & 360 \\
k/p/d           & deterministic rule (= oracle) & 1.00 & 1.00 & 0 \\
k/p/d           & per-node judge      & 1.00 & 0.44 & 0 \\
k/p/d           & structured, paraphrased & 1.00 & 0.90 [0.86, 0.95] & 0 \\
\bottomrule
\end{tabular}}
\end{table}

\subsection{CostBench.} CostBench~\cite{costbench} exposes six travel sub-tasks as refinement chains
$\textsf{Raw}\to L_1\to\dots\to L_5\to\textsf{Travel}$ with a Dijkstra solver (seed $42$, level
$5$, $210$ tools). Intent is
$(\textsf{task},\textsf{preferences},\textsf{depth})$: a \emph{full-swap} replaces the preferences (revision; all downstream candidates stale) and a \emph{stage-localised} change alters only the depth (addition or retraction; the prefix stays valid). Each of $120$ base queries receives each change type at an early, mid, or late timestep ($1{,}080$ instances, all passing a deterministic auditor); the solver gives the cached state, revised ground truth, $c_{\mathrm{restart}}$, and $c_{\mathrm{oracle}}$. In Table~\ref{tab:cost}, naive-continue answers the old goal on every instance (Success violation) and keep-all keeps stale results on all $360$ full-swap revisions, dropping to $67\%$ correctness (Inclusion violation). The \kpd{} policy with the deterministic rule matches the oracle at mean cost $67.95$ against restart's $119.15$; over a four-turn sequence (addition, revision, retraction, revision; $40$ queries) it stays correct at every turn at $58\%$ of cumulative restart cost, whereas keep-all's correctness collapses to $0$ at the first revision.

A per-node \texttt{gpt-4o} judge reading the natural-language revision is safe (zero over-salvage) but conservative, recovering $44\%$ of oracle savings. Its difficulty reduces to one question: did the revision change the target (discard all) or only the depth (keep the prefix)? A structured judge that asks this once recovers full savings on templated revisions ($60/60$) and $0.90$ when depth integers are hidden and revisions are paraphrased as mechanism-free utterances ($140/150$; every miss
is the ambiguous ``that's more than I need'' retraction read as a preference change, all in the safe direction). With \texttt{gpt-4o-mini} the structured judge still
scores $0.90$ $[0.85,0.94]$ while the per-node judge falls from $0.35$ to $0.00$.

\subsection{STT-Arena.} STT-Arena~\cite{sttarena} tasks are executable environments whose goal is a checklist of deterministic check-functions, read directly as the commitment hierarchy; retraction drops sub-goals and addition adds held-out ones ($1{,}110$ audited instances). With unchanged code, \kpd{} coincides with the oracle ($3.05$ sub-goals recomputed against restart's $6.08$, $100\%$
correct). For genuine revisions, a model rewrites one achieved sub-goal into a conflicting requirement following STT-Arena's conflict-authoring recipe ($150$ consistency-checked instances).
Here a per-node \texttt{gpt-4o} judge recovers $0.985$ $[0.97,1.0]$ of oracle savings at $99.3\%$ correctness but commits one Inclusion violation (keeps a stale sub-goal, answers wrongly), which never happened on CostBench; keep-all fails all $150$.
In real tool calls on \texttt{gpt-5} rollouts ($n{=}120$), keeping the cached environment with the full message history is no better than restart (pass rate $0.11$ in both), while keeping the environment state but resetting the context to what is done and what remains raises the pass rate to $0.52$ and cuts tool calls by $41\%$ on the both-pass subset ($n{=}12$).

\begin{table}[t]
\centering
\caption{Synthetic DAG ($188$ instances, $631$ nodes): savings recovered (over-salvage count).
Keep-all over-salvages $320$; keep/discard only, with perfect labels, recovers $0.76$ $[0.72,0.81]$.}
\label{tab:dag}
{\footnotesize
\begin{tabular}{@{}lccc@{}}
\toprule
\textbf{model} & \textbf{per-node} & \textbf{informed} & \textbf{structured} \\
\midrule
\texttt{gpt-4o}           & 0.28 (30) & 0.52 (2)  & \textbf{1.00 (0)} \\
\texttt{gpt-4o-mini}      & 0.00 (0)  & 0.60 (17) & \textbf{1.00 (0)} \\
\texttt{gemini-2.5-flash} & 0.04 (1)  & 0.15 (0)  & \textbf{1.00 (0)} \\
\bottomrule
\end{tabular}}
\end{table}

\subsection{Synthetic DAG.} Neither benchmark stresses \emph{patch} as distinct from discard, nor shared producers, so we add a synthetic dependency-DAG environment as a minimal sufficiency construction: typed ops with AND inputs and OR producers, revisions that make governed ops patchable, stale, or blocked, and an exact AND-OR solver ($80$ seeds, $188$ instances, $631$ nodes). Each op's one-line description says only how it used a business \emph{aspect} (formatting is patchable, computing is stale, a now-forbidden source is blocked) and the revision only that the aspect changed, so a judge must infer the action from the role. Here only patch is correct and cheap on a patchable node; keep/discard only, with perfect labels, leaves a quarter of the savings (Table~\ref{tab:dag}). The judge ladder matches the separation: the per-node judge over-salvages on the $33\%$ of instances where a node dies purely by propagation (Theorem~\ref{thm:locality}); the informed judge nearly closes the safety gap on \texttt{gpt-4o} ($30\to2$ violations) but, still deciding node by node, recovers about half the savings (the information suffices, the reasoning does not). Only the structured judge reaches the oracle, on all three models with zero violations. The other two columns are model-dependent and not monotone in model strength (informed-to-structured gap $0.48$, $0.40$, $0.85$ down the table).

\section{Scope and Limitations}
\label{sec:limits}

The role-design claim should be read at the scope of its evidence: three judge models from two vendors at temperature $0$, with the CostBench and STT-Arena judge results using only two OpenAI models; a wider sample and multiple seeds are needed before it can be stated for models in general. What it supports is that on every model tested, one classification plus deterministic propagation turned an unreliable per-node judge into an oracle-matching one. The method assumes a hierarchical plan with recorded provenance and deterministic propagation, which fail for plans revised while tool calls are in flight, side-effecting tools that discarding cannot undo, and latent dependencies the executor never records (treating undeclared dependence as dependence on everything then preserves Inclusion at the cost of Minimality). The synthetic DAG is an existence proof, not a prevalence claim, and revisions are templated or paraphrased rather than collected from users.

\section{Conclusion}
\label{sec:conclusion}

Casting the plan as a commitment hierarchy and the revision as belief revision gives postulates that name the failure modes of mid-task intent revision, a locality bound that explains why per-node judgment must fail, and an oracle to measure against. The practical lesson is role design: ask the model only what changed, and let a deterministic layer propagate it. Code, instance sets, and logs will be released.

\subsubsection*{Disclosure of Interests.}
The author has no competing interests to declare that are relevant to the content of this article.

\subsubsection*{AI Declaration.}
Generative AI tools were used for the experiments themselves, as the language-model judges under
study (\texttt{gpt-4o}, \texttt{gpt-4o-mini}, \texttt{gemini-2.5-flash}, \texttt{gpt-5}), and a
large language model assisted with copy-editing and language refinement of the manuscript. All
research design, code, analysis, and claims are the author's own; the author reviewed and verified
the entire text and takes full responsibility for its content.

\bibliographystyle{splncs04}
\bibliography{refs}

\end{document}